%% file: main.tex
\documentclass{article}

\usepackage[preprint]{neurips_2026}

\usepackage[utf8]{inputenc}
\usepackage[T1]{fontenc}
\usepackage[hidelinks]{hyperref}
\usepackage{url}
\usepackage{booktabs}
\usepackage{multirow}
\usepackage{array}
\usepackage{amsfonts}
\usepackage{amsmath}
\usepackage{amssymb}
\usepackage{amsthm}
\usepackage{mathtools}
\usepackage{microtype}
\usepackage{xcolor}
\usepackage{graphicx}
\usepackage{subcaption}
\usepackage{enumitem}
\usepackage{float}

\newcommand{\fastumap}{\textsc{FastUMAP}}
\newcommand{\umap}{\textsc{UMAP}}
\newcommand{\opentsne}{\textsc{openTSNE}}
\newcommand{\sude}{\textsc{SUDE}}
\newcommand{\R}{\mathbb{R}}

\newtheorem{proposition}{Proposition}

\title{FastUMAP: Scalable Dimensionality Reduction via Bipartite Landmark Sampling}

\author{%
Hongmin Li\textsuperscript{1,2}\\[0.25em]
\normalfont\textsuperscript{1}School of Life Science and Technology, Institute of Science Tokyo\\
\normalfont 2-12-1 Ookayama, Meguro-ku, Tokyo 152-8550, Japan\\
\normalfont\textsuperscript{2}Department of Computational Biology and Medical Sciences\\
\normalfont Graduate School of Frontier Sciences, The University of Tokyo\\
\normalfont 5-1-5 Kashiwanoha, Kashiwa-shi, Chiba 277-8561, Japan\\
\normalfont\texttt{lihongmin@edu.k.u-tokyo.ac.jp}\\[0.25em]
\normalfont\small Researcher, School of Life Science and Technology, Institute of Science Tokyo;\\
\normalfont\small Guest Researcher, Department of Computational Biology and Medical Sciences,\\
\normalfont\small Graduate School of Frontier Sciences, The University of Tokyo.\\
\normalfont\small ORCID: \href{https://orcid.org/0000-0003-0228-0600}{0000-0003-0228-0600}
}

\begin{document}

\maketitle

\begin{abstract}
\input{sections/abstract}
\end{abstract}

\section{Introduction}
\input{sections/introduction}

\section{Related Work}
\input{sections/related_work}

\section{Method}
\label{sec:method}
\input{sections/method}

\section{Experiments}
\input{sections/experiments}

\section{Limitations}
\input{sections/limitations}

\section{Conclusion}
\input{sections/conclusion}

\clearpage
\bibliographystyle{plainnat}
\bibliography{references}

\clearpage
\appendix
\section{Appendix}
\input{sections/appendix}

\input{checklist}

\end{document}

%% file: sections/abstract.tex
Exploratory analysis of high-dimensional data rarely stops at a single
embedding. In practice, analysts rerun dimensionality reduction after changing
preprocessing, subsets, or hyperparameters, and standard nonlinear methods can
quickly become the bottleneck. We introduce \textbf{FastUMAP} (Bipartite
Manifold Approximation and Projection), a landmark-based method designed for
this repeated-use setting. FastUMAP builds a sparse point-landmark fuzzy graph,
computes a Nystr\"{o}m spectral warm start from the induced landmark affinity,
and then refines all sample coordinates with a UMAP-style objective on the
bipartite graph. The landmark ratio $r = m/n$ provides a direct way to trade
runtime against fidelity. On 9 benchmark datasets spanning 178 to 70,000
samples, FastUMAP has the lowest runtime on 7 datasets in our reported
default-implementation comparison on one workstation. On MNIST and
Fashion-MNIST ($n{=}70{,}000$), it runs in about 4.6 seconds, compared with
about 73--75 seconds for Barnes--Hut t-SNE, while reaching 91.4\% mean kNN
accuracy versus 94.6\% for the strongest accuracy baseline. FastUMAP is
therefore best viewed as a fast option for repeated exploratory embedding,
rather than as a replacement for accuracy-first methods.

%% file: sections/introduction.tex
Nonlinear dimensionality reduction remains a standard tool for exploratory analysis, but the bottleneck changes once users need many embeddings rather than a single final visualization. In the workstation-scale regime studied here, one run can already take tens of seconds, so repeated parameter sweeps, preprocessing changes, and visual comparison become expensive. Barnes--Hut t-SNE and UMAP remain strong quality baselines, yet in our 70,000-sample benchmark they already fall into this latency range~\cite{maaten2014,mcinnes2018umap}.

The cost bottleneck is structural. Standard DR pipelines still rely on all-pairs or approximate all-pairs neighborhood graphs, which lead to $O(n \log n)$ or empirically superlinear behavior. Landmark-based methods offer a plausible escape by embedding only a subset of the data and projecting the rest, but existing methods often leave practitioners with an unclear trade-off: they are faster, yet it is difficult to tell what geometric information is being preserved, what initialization is appropriate, and which parameter actually controls the fidelity loss.

\begin{figure}[t]
\centering
\includegraphics[width=\linewidth]{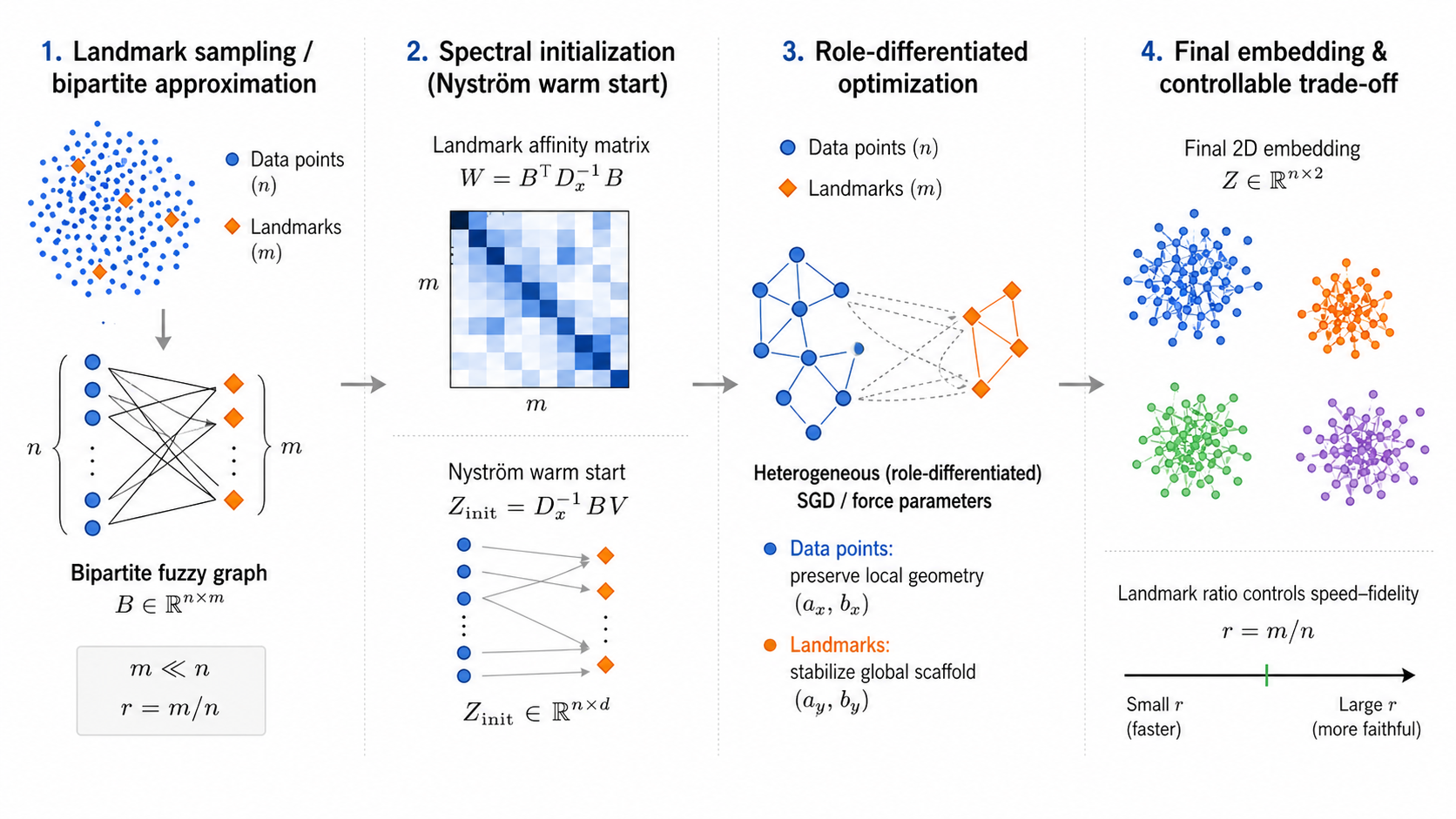}
\caption{\fastumap{} replaces the full-neighborhood graph bottleneck with an explicit landmark approximation. Standard UMAP-style pipelines build a neighborhood graph over all $n$ data points, making repeated exploratory runs slow. FastUMAP samples $m \ll n$ landmarks, builds a sparse bipartite fuzzy graph $\mathbf{B}\in\mathbb{R}^{n\times m}$, forms a landmark affinity $\mathbf{W}=\mathbf{B}^{\top}\mathbf{D}_x^{-1}\mathbf{B}$, computes a Nystr\"{o}m spectral warm start, and refines the embedding with role-differentiated optimization. The landmark ratio $r=m/n$ is the user-facing approximation knob: smaller $r$ favors speed, while larger $r$ increases neighborhood coverage and moves the method closer to full UMAP.}
\label{fig:main-conceptual-overview}
\end{figure}

We present \textbf{FastUMAP} (Bipartite Manifold Approximation and Projection), a landmark-based DR method for settings where embedding latency matters. FastUMAP keeps the local-connectivity semantics of UMAP, but replaces the full point-point graph with a sparse bipartite graph over samples and landmarks. It then uses a Nystr\"{o}m spectral warm start and a role-differentiated SGD refinement. The landmark ratio $r = m/n$ makes the approximation level explicit.
Figure~\ref{fig:main-conceptual-overview} summarizes this design:
the method trades a full graph for a bipartite landmark graph, keeps a spectral
initialization through a Nystr\"{o}m construction, and makes the speed--fidelity
choice explicit through $r$.

Across 9 benchmark datasets spanning 178 to 70,000 samples, FastUMAP has the lowest runtime on 7 datasets in our reported default-implementation comparison on one workstation and remains much faster than BH-t-SNE on the medium-to-large datasets in that environment. The price is a moderate drop in kNN accuracy relative to the strongest accuracy baseline, with the largest gaps on MNIST and Fashion-MNIST under the 5{,}000-landmark cap. We therefore present FastUMAP as a speed-first method with a clear fidelity control, not as a universal accuracy winner.

Our contributions are:
\begin{itemize}[leftmargin=1.2em]
    \item \textbf{A UMAP-inspired bipartite landmark formulation.} We replace dense all-pairs graph construction with a sparse bipartite fuzzy graph and initialize the layout with a Nystr\"{o}m spectral warm start.
    \item \textbf{A simple approximation control.} The landmark ratio $r = m/n$ provides a direct way to trade runtime against fidelity, while role-differentiated optimization is treated as a secondary design choice and evaluated empirically.
    \item \textbf{An empirical study of the trade-off.} On 9 benchmark datasets, we show where FastUMAP is fast and where it gives up accuracy in a reported default-implementation comparison. A supporting biological example is deferred to the appendix.
\end{itemize}

%% file: sections/related_work.tex
\label{sec:related}

\subsection{Graph-Based Nonlinear Dimensionality Reduction}

Modern nonlinear DR methods are dominated by two coupled costs: constructing a neighborhood graph over all $n$ samples and optimizing a low-dimensional layout over that graph. t-SNE~\cite{vandermaaten2008tsne} defines high-dimensional similarities over all pairs and optimizes a KL-divergence objective; Barnes--Hut t-SNE~\cite{maaten2014} and interpolation-based variants such as FIt-SNE~\cite{linderman2019} reduce the layout cost substantially, but the pipeline remains oriented toward high-fidelity embeddings rather than repeated low-latency runs. UMAP~\cite{mcinnes2018umap} replaces the t-SNE objective with a fuzzy simplicial-set construction and an edge-sampled cross-entropy objective, yielding a strong practical balance of quality and speed. It has therefore become a common default in exploratory workflows, including single-cell analysis~\cite{becht2019umap}. Even so, its graph-construction stage still scales with an all-point neighborhood search, which becomes expensive once the user needs many runs on datasets in the tens-of-thousands regime.

Other recent methods, including LargeVis~\cite{tang2016}, PaCMAP~\cite{wang2021pacmap}, and TriMAP~\cite{amid2019trimap}, explore alternative objectives or sampling schemes for improving the global-local trade-off. These methods are relevant conceptual comparators, but they do not directly address the specific design question studied here: how to expose a clear, user-facing approximation knob within a UMAP-style graph-and-SGD pipeline.

\subsection{Landmark and Anchor-Based Approximations}

Landmark methods reduce cost by replacing the full $n \times n$ neighborhood structure with a smaller representation defined by $m \ll n$ anchors or landmarks. Classical examples include Landmark MDS~\cite{desilva2003}, which embeds only the landmarks and triangulates the remaining points, and Nystr\"{o}m methods~\cite{williams2001,fowlkes2004}, which approximate large kernel matrices from a landmark subset. The common trade-off is clear: these methods reduce cost, but the approximation often becomes a post hoc projection step, so the relationship between the landmark construction and the final embedding objective is not always explicit.

More recent anchor-graph methods make this structure more explicit. In spectral clustering, bipartite anchor graphs can approximate the full affinity matrix while preserving enough structure for downstream eigendecomposition~\cite{li2022divideandconquer}. In dimensionality reduction, SUDE~\cite{peng2025sude} also uses landmarks to reduce cost, but it follows a different pattern from UMAP-style methods: landmarks are embedded under a modified t-SNE-like objective and non-landmarks are then reconstructed by a constrained projection step.

\subsection{Positioning of FastUMAP}

FastUMAP sits at the intersection of these two lines of work. Relative to UMAP, it replaces the full point-point neighborhood graph with a sparse point-landmark bipartite graph, so the approximation enters at graph construction rather than only through faster optimization. Relative to projection-style landmark methods, it does not treat non-landmark points as a purely post hoc reconstruction problem: all points participate in the same bipartite objective during optimization. Relative to anchor-graph spectral methods, the landmark affinity is not the final output of the method but a warm start for subsequent nonlinear refinement.

This combination is the main distinction of FastUMAP. The method keeps UMAP-style local-connectivity semantics, uses a Nystr\"{o}m-style spectral initialization derived from the bipartite graph, and exposes the landmark ratio $r = m/n$ as the primary speed--fidelity control. We therefore position FastUMAP not as a replacement for accuracy-first baselines, but as a speed-oriented member of the UMAP/landmark family whose approximation regime is explicit and tunable.

%% file: sections/method.tex
FastUMAP embeds a high-dimensional dataset $\mathcal{X} = \{\mathbf{x}_1, \ldots, \mathbf{x}_n\} \subset \R^D$ into $\R^2$ by compressing neighborhood construction through a small landmark set while still optimizing coordinates for all $n$ samples. In the variant studied in this paper, landmarks are sampled from the dataset itself, so ``landmark'' denotes a computational role assigned to a subset of samples rather than an additional set of free variables. The pipeline has four stages: (i) choose $m \ll n$ landmarks, (ii) build a sparse bipartite fuzzy graph from all samples to those landmarks, (iii) compute a reduced spectral warm start on the landmark side, and (iv) refine the layout by edge-sampled SGD with role-differentiated force profiles.

\subsection{Landmark Sampling}

Let $S = \{s_1, \ldots, s_m\} \subset [n]$ denote landmark indices and define $\mathbf{l}_p = \mathbf{x}_{s_p}$ for $p=1,\ldots,m$. The \emph{landmark ratio} $r = m/n$ is the main user-facing approximation knob. Large $r$ retains more neighborhood coverage and moves the method closer to full UMAP, whereas small $r$ reduces graph-construction and spectral cost at the expense of a coarser approximation.

The method itself only requires $\mathcal{L} \subset \mathcal{X}$; the reported experiments use the capped landmark schedules summarized in the appendix. This separation is important: FastUMAP is defined for arbitrary $m$, while the benchmark operating points correspond to particular speed-oriented budget choices.

\begin{proposition}[FastUMAP--UMAP equivalence]
\label{prop:fastumap_umap_equivalence}
When $m = n$, $\mathcal{L} = \mathcal{X}$, and the force parameters for data points and landmarks are tied, the FastUMAP objective reduces to standard UMAP.
\end{proposition}

The proof is given in Appendix~\ref{app:equivalence-proof}. In practice FastUMAP is used in the $m \ll n$ regime, where it should be viewed as an explicit approximation family with a controllable fidelity budget rather than as a disguised reimplementation of UMAP.

\subsection{Bipartite Fuzzy Graph Construction}

For each sample $\mathbf{x}_i$, we find its $k$ nearest landmarks and compute adaptive fuzzy memberships with the same local-connectivity calibration used in UMAP~\cite{mcinnes2018umap}. Let $\mathcal{N}_k(i)$ denote those landmark indices. For each $i$, we determine a local offset $\rho_i$ and bandwidth $\sigma_i$ such that
\begin{equation}
\label{eq:membership}
  w_{ip} =
  \begin{cases}
    \exp\!\Bigl(-\dfrac{\max(d(\mathbf{x}_i, \mathbf{l}_p) - \rho_i,\; 0)}{\sigma_i}\Bigr), & p \in \mathcal{N}_k(i), \\
    0, & \text{otherwise},
  \end{cases}
  \qquad
  \sum_{p \in \mathcal{N}_k(i)} w_{ip} = \log_2 k.
\end{equation}
Stacking these memberships yields a sparse bipartite matrix $\mathbf{B} \in \mathbb{R}^{n \times m}$ with at most $nk$ nonzeros.

Because every landmark is also an original sample, each nonzero bipartite affinity induces a pair of directed optimization edges on the sample index set. Writing $\pi(p)=s_p$ for the map from landmark slots to original sample indices, every nonzero $B_{ip}$ contributes $(i,\pi(p))$ and $(\pi(p),i)$ with the same sampling weight. This keeps attraction bidirectional during SGD while storing only the sparse bipartite affinities.

\subsection{Reduced Spectral Warm Start}

Rather than initializing all coordinates randomly, FastUMAP extracts a reduced spectral basis from the bipartite graph. Let
\begin{equation}
\label{eq:landmark_affinity}
  \mathbf{D}_x = \mathrm{diag}(\mathbf{B}\mathbf{1}),
  \qquad
  \mathbf{W} = \mathbf{B}^\top \mathbf{D}_x^{-1} \mathbf{B},
\end{equation}
and let $\mathbf{D}_\ell = \mathrm{diag}(\mathbf{W}\mathbf{1})$. The matrix $\mathbf{W}$ aggregates two-step sample-to-landmark-to-landmark transitions and can be interpreted as the landmark-side reduction of the bipartite normalized-cut problem. We then solve the $m \times m$ eigenproblem
\begin{equation}
\label{eq:normalized_landmark_operator}
  \mathbf{M} = \mathbf{D}_\ell^{-1/2} \mathbf{W} \mathbf{D}_\ell^{-1/2}
\end{equation}
and keep the two leading non-trivial eigenvectors in $\mathbf{U} \in \mathbb{R}^{m \times 2}$. Every sample is then assigned an initial coordinate by Nystr\"{o}m-style projection through its landmark weights:
\begin{equation}
\label{eq:nystrom}
  \mathbf{Z}_{\mathrm{init}} = \mathbf{D}_x^{-1} \mathbf{B}\,\mathbf{U}.
\end{equation}
This step makes the approximation explicit: only an $m \times m$ operator is diagonalized, yet every sample receives an initialization informed by the geometry of its landmark neighborhood. In practice, this warm start materially reduces the number of SGD epochs needed to reach a useful embedding quality level.

\subsection{Role-Differentiated Optimization}

The warm start is refined over a single coordinate set $\mathbf{Z} = \{\mathbf{z}_1, \ldots, \mathbf{z}_n\}$ by UMAP-style edge-sampled SGD with negative sampling~\cite{mcinnes2018umap}. The low-dimensional connection kernel is
\begin{equation}
\label{eq:lowdim_kernel}
  \phi_r(d) = \frac{1}{1 + a_r d^{2b_r}},
  \qquad r \in \{x, y\},
\end{equation}
where $(a_x,b_x)$ is used when the sampled edge is traversed with its head vertex in the ordinary data role, and $(a_y,b_y)$ is used when the same relation is traversed in the landmark role through the reverse directed edge. Thus the heterogeneity is attached to the \emph{role of the updated vertex}, not to a separate landmark embedding space.

Let $\mathcal{E}^+$ denote the directed edge multiset obtained from the nonzeros of $\mathbf{B}$ by the duplication rule above, and let positive edges be sampled with frequency proportional to their inherited weights. A negative-sampling approximation to the fuzzy cross-entropy objective is
\begin{equation}
\label{eq:sgd_objective}
  \widehat{\mathcal{L}}
  = - \sum_{(u,v) \in \mathcal{E}^+}
  \left[
    \log \phi_{r(u)}\!\bigl(\lVert \mathbf{z}_u - \mathbf{z}_v \rVert\bigr)
    + \gamma \; \mathbb{E}_{v^- \sim P_n}
    \log \Bigl(1 - \phi_{r(u)}\!\bigl(\lVert \mathbf{z}_u - \mathbf{z}_{v^-} \rVert\bigr)\Bigr)
  \right],
\end{equation}
where $P_n$ is the negative-sampling distribution over the $n$ sample vertices and $r(u) \in \{x,y\}$ records whether $u$ is currently acting in the data or landmark role. The attractive gradient applied to the head vertex of a positive edge is therefore
\begin{equation}
\label{eq:attractive_gradient}
  \mathbf{g}^{+}_{uv}
  =
  - \frac{2 a_r b_r \, \lVert \mathbf{z}_u - \mathbf{z}_v \rVert^{2b_r-2}}
         {1 + a_r \lVert \mathbf{z}_u - \mathbf{z}_v \rVert^{2b_r}}
    (\mathbf{z}_u - \mathbf{z}_v),
\end{equation}
with $r=r(u)$, and the repulsive term uses the same role-dependent parameters against randomly drawn negatives. This role asymmetry is the main algorithmic departure from standard UMAP: the same pair of samples can experience different force profiles depending on whether the current SGD step is updating a query sample or a landmark representative.

\subsection{Complexity}

For direct sample-to-landmark search, the dominant costs are bipartite $k$-NN construction, reduced spectral initialization, and SGD:
\begin{equation}
\label{eq:complexity}
  T_{\mathrm{FastUMAP}} = O(nmd + nk^2 + m^2 i + Enk),
\end{equation}
where $d$ is the input dimension, $i$ is the number of eigensolver iterations, and $E$ is the number of SGD epochs. The first term reflects distances from $n$ samples to $m$ landmarks, the second term forms $\mathbf{W}$ from a matrix with $O(nk)$ nonzeros, the third diagonalizes only an $m \times m$ operator, and the final term follows from optimizing a graph with $O(nk)$ positive edges and constant-rate negative sampling.

For fixed $m$, $k$, and $E$, the method is linear in $n$. In the capped-adaptive schedule used in our experiments, however, $m$ grows with $n$ until reaching the cap of 5{,}000, so the pre-cap regime is not strictly linear. Once the cap is active, the spectral term is effectively bounded and the observed runtime growth over datasets up to 70,000 samples is empirically close to linear. We therefore reserve the strict $O(n)$ claim for the fixed-budget regime and use the empirical near-linear description only for the capped setting studied in the paper.

%% file: sections/experiments.tex
\subsection{Evaluation Protocol}
\label{subsec:setup}

Our evaluation asks three questions. First, how much runtime can FastUMAP save under ordinary public implementations rather than under a hand-tuned microbenchmark? Second, does the landmark ratio $r = m/n$ provide a predictable runtime--quality trade-off? Third, do the spectral warm start and role-differentiated optimization matter in practice beyond the underlying landmark approximation?

To answer these questions, we evaluate FastUMAP on 9 benchmark datasets spanning 178 to 70,000 samples. The suite covers low-dimensional tabular data (Wine, Dermatology, Breast Cancer), medium-scale structured benchmarks (Mfeat, Spambase, Dry Bean, Shuttle), and large image datasets (MNIST and Fashion-MNIST). All methods receive identical preprocessing: min-max normalization followed by PCA to $\min(50, d_{\mathrm{orig}})$ when $d_{\mathrm{orig}} > 50$ or $n > 5{,}000$. We compare against four baselines: \textbf{BH-t-SNE}~\cite{maaten2014}, \textbf{\opentsne}~\cite{poličar2019opentsne}, \textbf{\umap}~\cite{mcinnes2018umap}, \textbf{\sude}~\cite{peng2025sude}. FastUMAP uses the submitted implementation defaults together with the adaptive landmark budgets listed in the appendix.

We report two primary quantities. The quality metric is kNN accuracy in the embedding space, which matches the local-neighborhood preservation claim made by UMAP-style methods and by FastUMAP itself. The runtime metric is wall-clock dimensionality-reduction time after the shared preprocessing step, because the intended use case is repeated exploratory execution on a workstation. The main benchmark tables are frozen from a cleaned 9-dataset rerun suite: runtime values are medians over three warm-cache runs per method on already preprocessed inputs, whereas quality values remain fixed-seed point estimates on the resulting embeddings. We treat the main runtime table as a reported default-implementation comparison, not as a fully thread-normalized microbenchmark. A matched single-thread diagnostic against BH-t-SNE is reported separately in the appendix to isolate the core runtime advantage without changing the main benchmark setup.

\subsection{Benchmark Results}
\label{subsec:benchmark_results}

Table~\ref{tab:knn_accuracy} summarizes the quality side of this trade-off. BH-t-SNE remains the strongest pure-accuracy baseline in the suite, reaching 94.6\% mean kNN accuracy, while FastUMAP reaches 91.4\%. We treat this gap as the cost of the landmark approximation rather than as an implementation defect. The mean alone is less informative than the per-dataset pattern: on several structured medium-scale datasets, such as Shuttle, FastUMAP remains nearly tied in accuracy (99.5\% vs.\ 99.7\%) while running much faster, whereas the largest gaps appear on MNIST and Fashion-MNIST, where the 5{,}000-landmark cap is most restrictive.

Table~\ref{tab:runtime} summarizes the latency side. Under each method's standard public implementation on the same workstation, FastUMAP has the lowest runtime on 7 of the 9 datasets. On the medium-to-large datasets in this environment, it runs substantially faster than BH-t-SNE, and on MNIST and Fashion-MNIST ($n{=}70{,}000$) it finishes in about 4.6 seconds, compared with about 73--75 seconds for BH-t-SNE. Appendix Table~\ref{tab:single_thread} asks a narrower question: if the comparison is restricted to matched single-thread settings against BH-t-SNE, does the speed advantage remain? It does. Taken together, these two views support a limited claim: FastUMAP is useful for latency-sensitive exploratory work, but the paper does not claim an implementation-agnostic ranking across every package and threading configuration.

\input{tables/table1_knn_accuracy}
\input{tables/table3_runtime}

\begin{figure}[t]
\centering
\includegraphics[width=\textwidth]{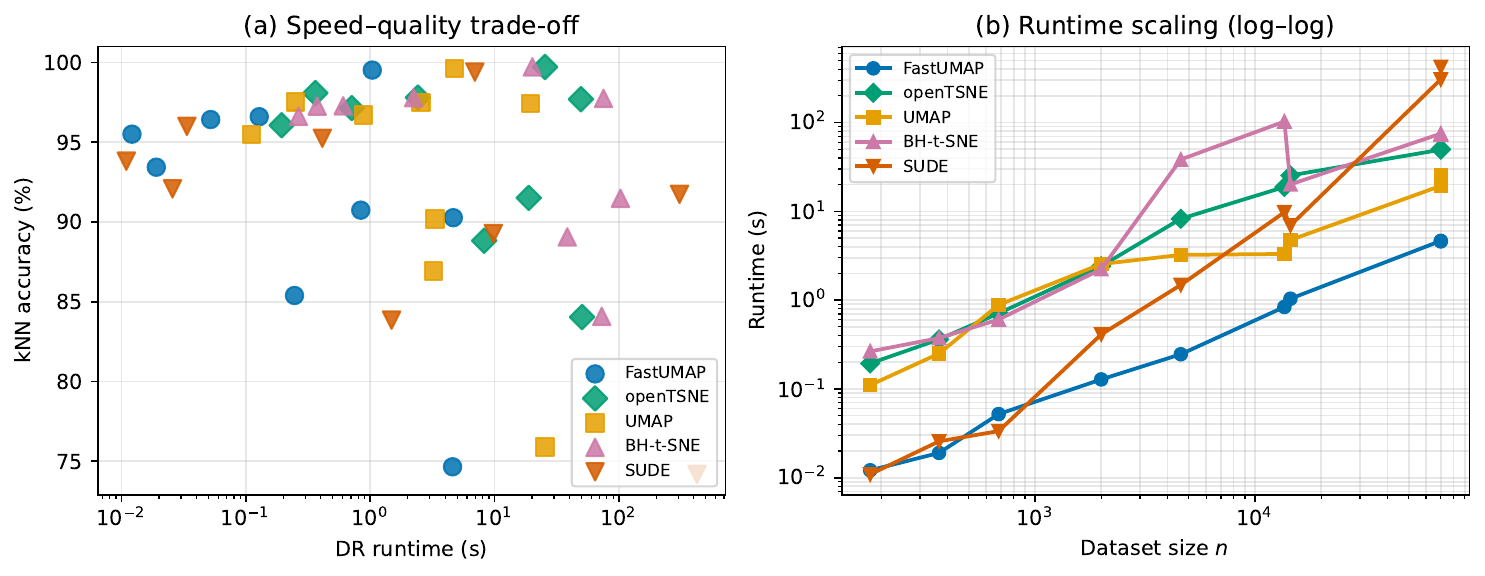}
\caption{\textbf{Quantitative comparison of FastUMAP against standard baselines.} FastUMAP sits in the fast part of the accuracy--runtime trade-off while retaining much of the neighborhood quality on the medium-to-large datasets.}
\label{fig:quantitative}
\end{figure}

\subsection{Ablation and Controllable Trade-off}
\label{subsec:ablation}

The ablations are designed to answer three separate questions. First, is the landmark ratio $r = m/n$ the dominant runtime--fidelity control? Second, does the spectral stage improve optimization dynamics beyond what could be obtained from a random start? Third, do heterogeneous force parameters contribute anything beyond the underlying landmark approximation?

Figure~\ref{fig:r_sweep} answers the first question. As $r$ increases, accuracy improves smoothly at higher runtime cost, confirming that the landmark budget is the main control on approximation quality. FastUMAP is therefore better understood as a family of operating points indexed by $r$, not as a single fixed approximation to UMAP.

Figure~\ref{fig:spectral_init} and Table~\ref{tab:ablation} address the second question. Spectral initialization acts mainly as an optimization accelerator: it helps the method reach useful quality earlier, rather than moving the eventual converged optimum by a large margin. This matches the role described in Section~\ref{sec:method}: the spectral stage provides a geometry-aware warm start rather than a separate embedding objective.

Table~\ref{tab:ablation} addresses the third question. Heterogeneous force parameters improve accuracy on 3 of 4 ablation datasets, but the gains are modest and dataset-dependent. We therefore treat this component as a helpful secondary design choice rather than as the main reason the method works.

\input{tables/table_ablation}

\begin{figure}[t]
\centering
\includegraphics[width=0.92\linewidth]{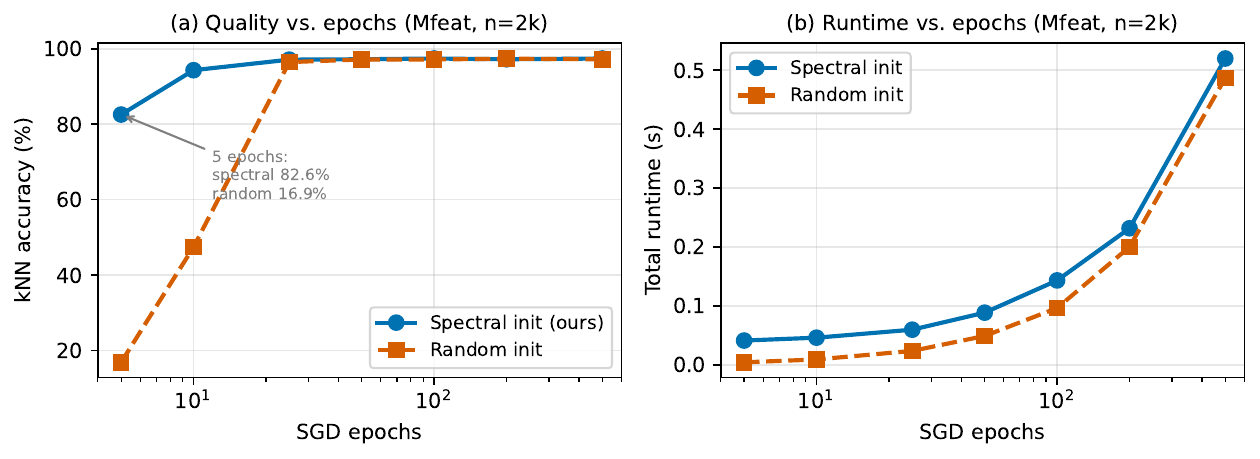}
\caption{\textbf{Spectral versus random initialization on Mfeat.} Spectral initialization reaches useful accuracy much earlier than random initialization, which matters when embeddings need to be produced quickly.}
\label{fig:spectral_init}
\end{figure}

\begin{figure}[t]
\centering
\includegraphics[width=0.95\linewidth]{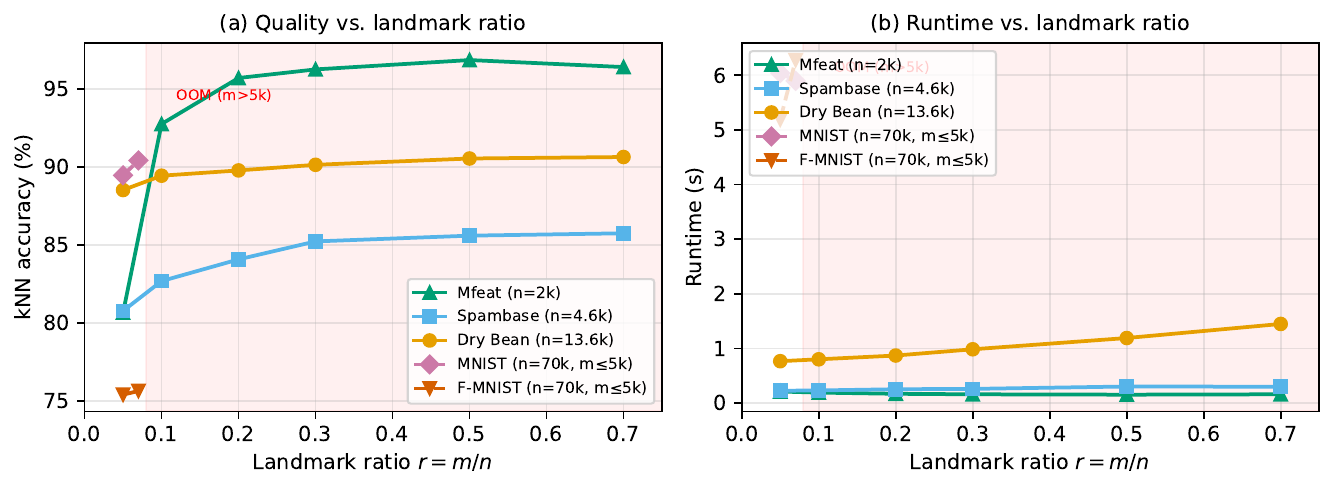}
\caption{\textbf{Runtime--quality trade-off induced by the landmark ratio $r = m/n$.} Accuracy improves smoothly as more landmarks are retained, at a moderate runtime cost.}
\label{fig:r_sweep}
\end{figure}

\noindent A supporting biological example is included in the appendix to test whether the same approximation remains useful outside the generic benchmark suite; we do not treat that example as part of the main benchmark average.

%% file: tables/table1_knn_accuracy.tex
\begin{table}[t]
\centering
\caption{\textbf{kNN classification accuracy} ($k{=}5$, 5-fold CV) across 9 benchmark datasets.
Best per row in \textbf{bold}; second-best \underline{underlined}.}
\label{tab:knn_accuracy}
\small
\begin{tabular}{@{}lccccc@{}}
\toprule
\textbf{Dataset} & \textbf{FastUMAP} & \textbf{openTSNE} & \textbf{UMAP} & \textbf{BH-t-SNE} & \textbf{SUDE} \\
\midrule
Wine & 95.5 & \underline{96.1} & 95.5 & \textbf{96.6} & 93.8 \\
Dermatology & 93.4 & \textbf{98.1} & \underline{97.5} & 97.3 & 92.1 \\
Breast Cancer & 96.4 & \underline{97.1} & 96.7 & \textbf{97.3} & 96.0 \\
Mfeat & 96.6 & \textbf{97.8} & 97.5 & \textbf{97.8} & 95.2 \\
Spambase & 85.4 & \underline{88.8} & 86.9 & \textbf{89.1} & 83.9 \\
Dry Bean & 90.7 & \textbf{91.5} & 90.2 & \underline{91.5} & 89.3 \\
Shuttle & 99.5 & \underline{99.7} & 99.6 & \textbf{99.7} & 99.4 \\
MNIST & 90.3 & \underline{97.7} & 97.4 & \textbf{97.8} & 91.7 \\
F-MNIST & 74.7 & \underline{84.0} & 75.9 & \textbf{84.1} & 74.2 \\
\bottomrule
\end{tabular}
\end{table}

%% file: tables/table3_runtime.tex
\begin{table}[t]
\centering
\caption{\textbf{DR runtime medians (seconds)} across 9 benchmark datasets under each method's standard public implementation on the same workstation, measured after the shared preprocessing step described in Section~\ref{subsec:setup}. Values are medians over three warm-cache runs per method. FastUMAP and openTSNE use multi-core execution; UMAP uses \texttt{n\_jobs=-1}; BH-t-SNE and SUDE are single-threaded in the evaluated implementations. Best time per row in \textbf{bold}.}
\label{tab:runtime}
\small
\begin{tabular}{@{}lrrrrrr@{}}
\toprule
\textbf{Dataset} & $n$ & \textbf{FastUMAP} & \textbf{openTSNE} & \textbf{UMAP} & \textbf{BH-t-SNE} & \textbf{SUDE} \\
\midrule
Wine & 178 & 0.012 & 0.194 & 0.110 & 0.264 & \textbf{0.011} \\
Dermatology & 366 & \textbf{0.019} & 0.361 & 0.250 & 0.374 & 0.026 \\
Breast Cancer & 699 & 0.052 & 0.710 & 0.884 & 0.605 & \textbf{0.033} \\
Mfeat & 2\,000 & \textbf{0.128} & 2.42 & 2.55 & 2.23 & 0.411 \\
Spambase & 4\,601 & \textbf{0.246} & 8.25 & 3.23 & 38.43 & 1.48 \\
Dry Bean & 13\,611 & \textbf{0.838} & 18.83 & 3.31 & 102.85 & 9.77 \\
Shuttle & 14\,500 & \textbf{1.04} & 25.42 & 4.76 & 20.11 & 6.94 \\
MNIST & 70\,000 & \textbf{4.65} & 49.49 & 19.39 & 75.13 & 307.14 \\
F-MNIST & 70\,000 & \textbf{4.59} & 50.53 & 25.41 & 72.75 & 424.23 \\
\bottomrule
\end{tabular}
\end{table}

%% file: tables/table_ablation.tex
\begin{table}[t]
\centering
\caption{\textbf{Ablation study.} Effect of initialization method and force parameter strategy on kNN accuracy (\%) and runtime (seconds). ``Spectral'' is the default bipartite Laplacian initialization; ``Random'' uses random coordinates. ``Hetero'' uses separate $(a_x, b_x)$ and $(a_y, b_y)$ for data and landmarks (default); ``Homo'' uses identical parameters for both.}
\label{tab:ablation}
\small
\begin{tabular}{@{}l cc cc@{}}
\toprule
& \multicolumn{2}{c}{\textbf{Initialization}} & \multicolumn{2}{c}{\textbf{Force Parameters}} \\
\cmidrule(lr){2-3} \cmidrule(lr){4-5}
\textbf{Dataset} & Spectral & Random & Hetero (default) & Homo \\
\midrule
Mfeat      & 97.0 / 1.0s & 97.1 / 0.2s & \textbf{97.3} / 0.2s & 97.0 / 0.2s \\
Dry Bean   & 86.3 / 2.1s & 86.4 / 1.3s & \textbf{87.0} / 1.2s & 86.4 / 1.2s \\
Shuttle    & 99.5 / 1.5s & 99.5 / 1.5s & \textbf{99.6} / 1.5s & 99.5 / 1.5s \\
F-MNIST    & 77.0 / 13.4s & 77.1 / 13.5s & 77.1 / 13.8s & \textbf{78.0} / 14.0s \\
\bottomrule
\end{tabular}

\vspace{2pt}
{\footnotesize Values: kNN accuracy (\%) / runtime (s). Bold indicates best accuracy per row.}
\end{table}

%% file: sections/limitations.tex
FastUMAP does not dominate across all regimes. BH-t-SNE remains the strongest pure-accuracy baseline in the evaluated suite, and UMAP remains a strong general-purpose default when wall-clock latency is not the primary constraint. The present results therefore support a speed-first claim rather than a claim of overall dominance.

The approximation cost is most visible at the largest scales. On MNIST and Fashion-MNIST, the fixed 5{,}000-landmark cap creates a clear accuracy gap relative to BH-t-SNE. This cap is an engineering choice that keeps the spectral stage feasible on commodity hardware; increasing the landmark ratio improves quality, but correspondingly moves the method away from its fastest operating regime.

The empirical scope is also selective rather than exhaustive. We focus on baselines reproduced under one environment and one preprocessing protocol, which yields a cleaner core comparison but leaves room for broader comparisons against additional modern DR packages. The quality tables remain fixed-seed point estimates rather than uncertainty-aware summaries, and the runtime table prioritizes practitioner-facing default implementations over a fully thread-normalized benchmark even though the frozen rerun suite includes repeated warm-cache timing for the main methods. In particular, we only provide a matched single-thread diagnostic against BH-t-SNE, not a full thread-normalized audit for UMAP or openTSNE.

Lower-latency dimensionality reduction can modestly broaden access to exploratory analysis on commodity hardware, especially for scientific users without large compute budgets. The corresponding risk is interpretive: in scientific or biomedical workflows, low-dimensional embeddings can be over-interpreted as decision tools. FastUMAP should therefore be treated as an exploratory visualization method rather than a standalone basis for high-stakes decisions.

%% file: sections/conclusion.tex
We introduced FastUMAP, a UMAP-inspired bipartite landmark method for fast nonlinear dimensionality reduction. The method makes its main trade-off explicit: it replaces the full neighborhood graph with a sparse point-landmark graph, uses a Nystr\"{o}m spectral warm start to initialize the layout, and exposes the landmark ratio $r = m/n$ as the main approximation control. Across 9 benchmark datasets in our reported default-implementation comparison on one workstation, FastUMAP has the lowest runtime on 7 of the 9 datasets and performs especially well on medium-to-large problems where embedding latency matters.

The results point to a practical use case rather than a universal winner. FastUMAP is most useful when repeated runs, interactive exploration, or latency-sensitive workflows matter more than extracting the last few points of neighborhood accuracy. In accuracy-first settings, BH-t-SNE remains the stronger baseline and UMAP remains a strong general-purpose default. The most natural next steps are to improve large-scale fidelity under tight landmark budgets, strengthen landmark selection beyond the current default policy, and broaden the runtime evaluation across additional public implementations.

%% file: sections/appendix.tex
\subsection{Proof of Proposition~\ref{prop:fastumap_umap_equivalence}}
\label{app:equivalence-proof}

\begin{proof}
When $\mathcal{L} = \mathcal{X}$, every data point is simultaneously a landmark. The bipartite $k$-NN search from $\mathbf{x}_i$ to its $k$ nearest landmarks is then exactly the standard $k$-NN search among all points, so the edge set matches UMAP's directed neighborhood graph.

The fuzzy membership weights are also identical. The local offset $\rho_i$ becomes the distance to the first nearest neighbor in $\mathcal{X}$, and the bandwidth $\sigma_i$ is calibrated by the same binary-search rule used in UMAP. Applying the same fuzzy-set symmetrization therefore produces the same symmetric affinity matrix.

Finally, when the force parameters for data points and landmarks are tied, the attractive and repulsive updates coincide with UMAP's standard low-dimensional objective. The only remaining difference is initialization: FastUMAP uses a bipartite spectral warm start, whereas UMAP often starts from a random or spectral layout on the full graph. Given enough optimization steps, both procedures optimize the same objective.
\end{proof}

\subsection{Frozen benchmark provenance}

The submission tables and figures are frozen from a cleaned 9-dataset rerun suite included in the anonymous supplementary package. The authoritative result snapshots are the CSV files \texttt{results\_v3.csv}, \texttt{results\_openTSNE\_multithread.csv}, \texttt{results\_single\_thread.csv}, \texttt{results\_r\_sweep.csv}, and \texttt{results\_spectral\_init.csv}. In particular, \texttt{results\_v3.csv} stores \texttt{dr\_time\_median} values over three warm-cache executions per method on the already preprocessed inputs. FastUMAP, UMAP, BH-t-SNE, and SUDE runtimes are stabilized with repeated warm-cache executions in that rerun suite; openTSNE is reported from the corresponding multithreaded rerun snapshot. The shared min-max normalization and optional PCA step is applied before these timings begin. Quality metrics remain fixed-seed 5-fold-CV point estimates on the resulting embeddings, so we do not treat them as uncertainty-aware summaries.

\subsection{Per-dataset preprocessing and landmark budgets}

Table~\ref{tab:per_dataset_m} records the exact input size, feature dimension, and landmark budget used for each dataset in the reported benchmark suite. We include these details because the capped landmark schedule is part of the experimental protocol: small datasets use a fixed landmark fraction, while medium and large datasets use the same fraction until the cap of 5,000 landmarks is reached. This makes the large-dataset runs speed-oriented by design and should be interpreted as a controlled operating point rather than an accuracy-maximizing sweep.

The preprocessing column also separates dimensionality reduction from the FastUMAP algorithm itself. All methods receive the same min-max normalized input, and PCA is applied only as a shared benchmark preprocessing step before any method-specific runtime is measured. Datasets with low original dimension but many samples therefore show a ``no-op'' PCA trigger: the preprocessing rule is activated by sample count, but the feature dimension is already below the PCA target.

\input{tables/table_per_dataset_m}

\subsection{Experimental defaults and compute environment}

All benchmark tables in Section~\ref{subsec:benchmark_results} are tied to the frozen rerun snapshots described above. Table~\ref{tab:hyperparams} lists the FastUMAP input budgets after the shared preprocessing step, including the landmark ratio $r=m/n$ used by each dataset. These ratios explain the main speed--quality operating points in the paper: the smaller datasets keep dense landmark coverage, whereas MNIST and Fashion-MNIST use a much smaller ratio after the landmark cap is reached.

The anonymous supplementary package contains the corresponding code, configs, lightweight sanity-check scripts, and frozen result files. Runtime measurements were collected on an Apple M4 Max workstation with 36\,GB RAM. The main benchmark tables report method runtime after shared preprocessing; Table~\ref{tab:runtime_breakdown} separately isolates FastUMAP stage costs under single-threaded Numba, while Table~\ref{tab:single_thread} provides a threading-matched comparison against BH-t-SNE.

\input{tables/table_hyperparams}

\clearpage
\subsection{Threading-matched runtime comparison}

Table~\ref{tab:single_thread} isolates the runtime comparison most sensitive to implementation threading: FastUMAP is run once with Numba restricted to a single thread and once with the workstation default, while BH-t-SNE remains single-threaded.

\input{tables/supp_table_single_thread}
\clearpage

\subsection{FastUMAP runtime breakdown}

Table~\ref{tab:runtime_breakdown} decomposes FastUMAP into graph construction, spectral initialization, and SGD refinement so that the total runtime claims can be traced to concrete pipeline stages.

\input{tables/table_runtime_breakdown}
\clearpage

\subsection{Biological application example}
\label{app:scrna}

We also include a biological case study to test whether the approximation remains useful outside the generic benchmark suite. On a wild-type mouse retina scRNA-seq dataset with 6,301 cells across 13 retinal cell types~\cite{heng2019,wolf2018scanpy}, FastUMAP reaches 89.0\% kNN accuracy, compared with 89.8\% for UMAP and 90.5\% for BH-t-SNE. These results indicate that the same approximation can preserve meaningful structure in a domain where UMAP and BH-t-SNE are commonly used.

\begin{figure}[H]
\centering
\includegraphics[width=\textwidth]{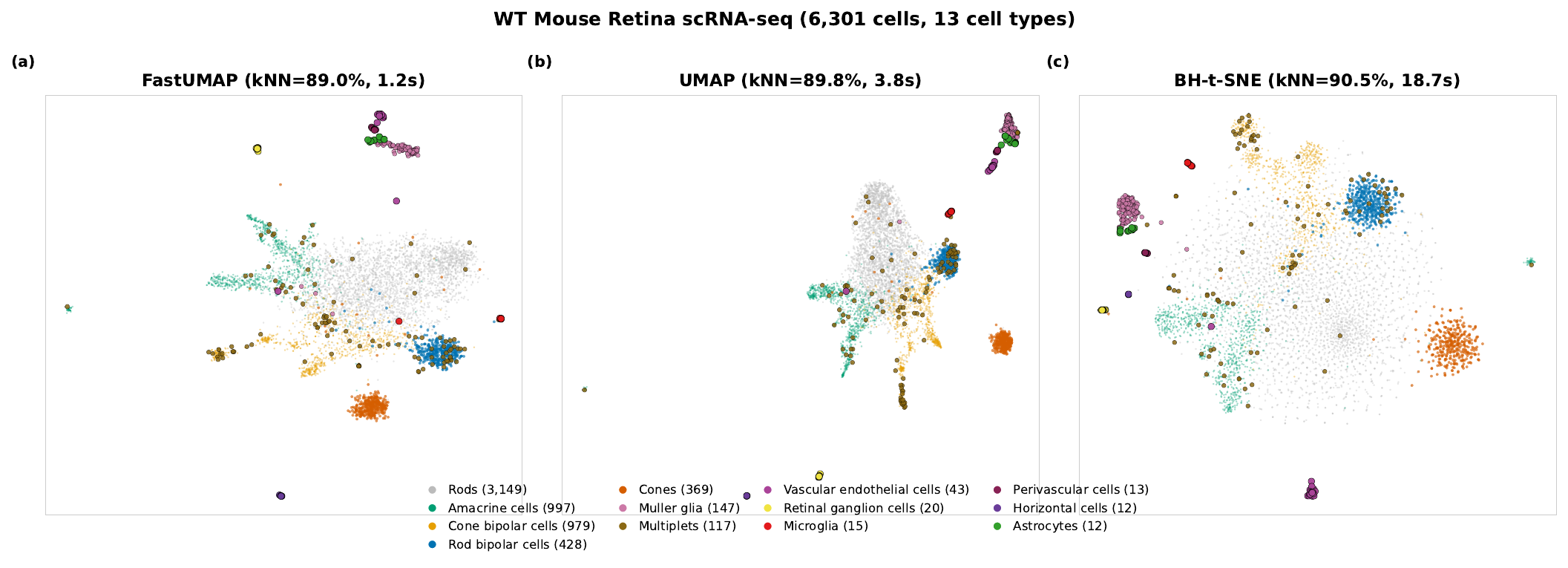}
\caption{\textbf{Single-cell RNA-seq embedding example.} FastUMAP preserves the major retinal cell populations while remaining close to UMAP and BH-t-SNE on downstream kNN accuracy.}
\label{fig:scrna}
\end{figure}

%% file: tables/table_per_dataset_m.tex
\begin{table}[t]
\centering
\caption{\textbf{Per-dataset experimental configuration.} $d_{\text{orig}}$ = original feature count; $d$ = input dimension after preprocessing; $m$ = landmark count used in experiments. Preprocessing: min-max normalization, then PCA to $\min(50, d_{\text{orig}})$ dimensions when $d_{\text{orig}} > 50$ or $n > 5{,}000$ (effective only when $d_{\text{orig}} > 50$). Landmark rule: $m = 0.5n$ for $n < 500$; $m = 0.7n$ for $500 \leq n < 5{,}000$; $m = \min(0.7n, 5{,}000)$ for $n \geq 5{,}000$.}
\label{tab:per_dataset_m}
\small
\begin{tabular}{@{}lrrrrc@{}}
\toprule
\textbf{Dataset} & \textbf{$n$} & \textbf{$d_{\text{orig}}$} & \textbf{$d$} & \textbf{$m$} & \textbf{PCA} \\
\midrule
Wine        &      178 &   13 &  13 &     89 & — \\
Dermatology &      366 &   34 &  34 &    183 & — \\
Breast      &      699 &    9 &   9 &    489 & — \\
Mfeat       &   2\,000 &  649 &  50 &  1\,400 & 649→50 \\
Spambase    &   4\,601 &   57 &  50 &  3\,220 & 57→50 \\
Dry Bean    &  13\,611 &   16 &  16 &  5\,000 & no-op \\
Shuttle     &  14\,500 &    9 &   9 &  5\,000 & no-op \\
MNIST       &  70\,000 &  784 &  50 &  5\,000 & 784→50 \\
F-MNIST     &  70\,000 &  784 &  50 &  5\,000 & 784→50 \\
\bottomrule
\end{tabular}

\vspace{2pt}
{\footnotesize ``no-op'' = PCA triggered by $n > 5{,}000$ but $d_{\text{orig}} \leq 50$, so output dimension equals input.
For Dry Bean, Shuttle, MNIST, and F-MNIST, the formula $0.7n$ would give $m > 5{,}000$ (e.g., $0.7 \times 14{,}500 = 10{,}150$ for Shuttle), but the experimental cap of 5{,}000 applies; the library default has no hard cap.}
\end{table}

%% file: tables/table_hyperparams.tex
\begin{table}[t]
\centering
\caption{\textbf{Per-dataset FastUMAP input budgets.} The benchmark script supplies the landmark count $m$ shown here together with \texttt{random\_state=42} on Euclidean inputs after the preprocessing in Table~\ref{tab:per_dataset_m}. The ratio $r = m/n$ is the primary user-facing speed--quality control. Other optimization details follow the submitted implementation defaults; the frozen rerun provenance is described in the appendix and anonymous supplementary CSV snapshots.}
\label{tab:hyperparams}
\small
\begin{tabular}{@{}lrrrrc@{}}
\toprule
\textbf{Dataset} & \textbf{$n$} & \textbf{$d$} & \textbf{$m$} & \textbf{$r$} & \textbf{Distance} \\
\midrule
Wine        &    178 &  13 &     89 & 0.50 & Euclidean \\
Dermatology &    366 &  34 &    183 & 0.50 & Euclidean \\
Breast      &    699 &   9 &    489 & 0.70 & Euclidean \\
Mfeat       &  2\,000 &  50 &  1\,400 & 0.70 & Euclidean \\
Spambase    &  4\,601 &  50 &  3\,220 & 0.70 & Euclidean \\
Dry Bean    & 13\,611 &  16 &  5\,000 & 0.37 & Euclidean \\
Shuttle     & 14\,500 &   9 &  5\,000 & 0.34 & Euclidean \\
MNIST       & 70\,000 &  50 &  5\,000 & 0.07 & Euclidean \\
F-MNIST     & 70\,000 &  50 &  5\,000 & 0.07 & Euclidean \\
\bottomrule
\end{tabular}

\vspace{2pt}
{\footnotesize $d$ is the feature dimensionality after PCA (where applied). MNIST and F-MNIST are PCA-reduced from 784 to 50 dimensions.}
\end{table}

%% file: tables/supp_table_single_thread.tex
\begin{table}[H]
\centering
\caption{\textbf{Single-threaded runtime comparison.} FastUMAP with \texttt{NUMBA\_NUM\_THREADS=1} (matched to BH-t-SNE's single-threaded execution) versus its default multi-core setting and BH-t-SNE. The single-threaded speedup column provides a fair threading-matched comparison.}
\label{tab:single_thread}
\small
\begin{tabular}{@{}lrrrrc@{}}
\toprule
\textbf{Dataset} & \textbf{$n$} & \textbf{FastUMAP (1T)} & \textbf{FastUMAP (14T)} & \textbf{BH-t-SNE (1T)} & \textbf{Speedup (1T)} \\
\midrule
Wine & 178 & 0.02 & 0.01 & 0.26 & 15$\times$ \\
Dermatology & 366 & 0.04 & 0.02 & 0.37 & 9$\times$ \\
Breast & 699 & 0.21 & 0.05 & 0.61 & 3$\times$ \\
Mfeat & 2\,000 & 0.61 & 0.13 & 2.23 & 4$\times$ \\
Spambase & 4\,601 & 1.06 & 0.25 & 38.43 & 36$\times$ \\
Dry Bean & 13\,611 & 4.09 & 0.84 & 102.85 & 25$\times$ \\
Shuttle & 14\,500 & 4.44 & 1.04 & 20.11 & 5$\times$ \\
MNIST & 70\,000 & 30.13 & 4.65 & 75.13 & 2$\times$ \\
F-MNIST & 70\,000 & 25.75 & 4.59 & 72.75 & 3$\times$ \\
\bottomrule
\end{tabular}
\vspace{2pt}
{\footnotesize 1T = single-threaded (\texttt{NUMBA\_NUM\_THREADS=1}); 14T = default (Apple M4 Max, 14-core CPU). BH-t-SNE uses Barnes--Hut tree with single thread. Single-threaded speedups remain 2--36$\times$ across all datasets, confirming FastUMAP's efficiency advantage is not solely attributable to multi-core parallelism.}
\end{table}

%% file: tables/table_runtime_breakdown.tex
\begin{table}[H]
\centering
\caption{\textbf{FastUMAP runtime breakdown} (seconds) by pipeline stage, median of 3 warm-cache runs. ``Graph'' = bipartite $k$-NN construction. ``Spectral'' = $M{\times}M$ eigendecomposition + Nystr\"{o}m extension. ``SGD'' = stochastic gradient descent with negative sampling.}
\label{tab:runtime_breakdown}
\small
\begin{tabular}{@{}lrr rrrr@{}}
\toprule
\textbf{Dataset} & \textbf{$n$} & \textbf{$m$} & \textbf{Graph} & \textbf{Spectral} & \textbf{SGD} & \textbf{Total} \\
\midrule
Wine        &     178 &     89 & $<$0.01 & $<$0.01 & 0.01 & 0.01 \\
Dermatology &     366 &    183 & $<$0.01 & $<$0.01 & 0.02 & 0.02 \\
Breast      &     699 &    489 & $<$0.01 & $<$0.01 & 0.04 & 0.05 \\
Mfeat       &  2\,000 &  1\,400 & 0.01 & 0.01 & 0.17 & 0.20 \\
Spambase    &  4\,601 &  3\,220 & 0.02 & 0.03 & 0.24 & 0.31 \\
Dry Bean    & 13\,611 &  5\,000 & 0.11 & 0.24 & 0.54 & 0.91 \\
Shuttle     & 14\,500 &  5\,000 & 0.10 & 0.29 & 0.56 & 0.98 \\
MNIST       & 70\,000 &  5\,000 & 0.90 & 0.63 & 9.83 & 11.56 \\
F-MNIST     & 70\,000 &  5\,000 & 0.87 & 0.43 & 9.36 & 10.89 \\
\bottomrule
\end{tabular}

\vspace{2pt}
{\footnotesize $m$ = landmark count (adaptive rule in Table~\ref{tab:per_dataset_m}). Measured on Apple M4 Max (36\,GB RAM). Input preprocessed identically to main experiments. Breakdown timings were measured with Numba single-threaded mode (\texttt{NUMBA\_NUM\_THREADS=1}) to isolate per-stage costs; the main runtime table (Table~\ref{tab:runtime}) uses default multi-core Numba, yielding lower totals (e.g., 4.65s for MNIST vs 11.56s here).}
\end{table}

%% file: checklist.tex
\section*{NeurIPS Paper Checklist}

\begin{enumerate}

\item {\bf Claims}
    \item[] Question: Do the main claims made in the abstract and introduction accurately reflect the paper's contributions and scope?
    \item[] Answer: \answerYes{}
    \item[] Justification: Yes. The abstract and Section~1 describe FastUMAP as a speed-oriented dimensionality reduction method with an explicit speed--quality trade-off, and Section~5 states the same scope and non-dominance limits.

\item {\bf Limitations}
    \item[] Question: Does the paper discuss the limitations of the work performed by the authors?
    \item[] Answer: \answerYes{}
    \item[] Justification: Yes. Section~5 discusses the accuracy gap to BH-t-SNE, the effect of the landmark cap, the single-environment evaluation scope, the fixed-seed quality estimates, the practitioner-facing runtime framing, and the absence of a full thread-normalized audit for every baseline.

\item {\bf Theory assumptions and proofs}
    \item[] Question: For each theoretical result, does the paper provide the full set of assumptions and a complete (and correct) proof?
    \item[] Answer: \answerYes{}
    \item[] Justification: Yes. Section~3 states the assumptions for Proposition~\ref{prop:fastumap_umap_equivalence}, and the appendix provides the full proof of that equivalence result.

    \item {\bf Experimental result reproducibility}
    \item[] Question: Does the paper fully disclose all the information needed to reproduce the main experimental results of the paper to the extent that it affects the main claims and/or conclusions of the paper (regardless of whether the code and data are provided or not)?
    \item[] Answer: \answerYes{}
    \item[] Justification: Yes. Section~4 states the evaluation protocol, preprocessing, baselines, and seed convention, while the appendix and anonymous supplementary identify the frozen benchmark CSV snapshots and the accompanying implementation/configuration files that back the submission tables.

\item {\bf Open access to data and code}
    \item[] Question: Does the paper provide open access to the data and code, with sufficient instructions to faithfully reproduce the main experimental results, as described in supplemental material?
    \item[] Answer: \answerYes{}
    \item[] Justification: Yes. The anonymous supplementary material includes code, configs, frozen benchmark CSV snapshots for the submission figures/tables, and a guide that distinguishes those frozen results from the lightweight sanity-check scripts; the datasets are public benchmarks identified in Section~4 and the configuration files.

\item {\bf Experimental setting/details}
    \item[] Question: Does the paper specify all the training and test details (e.g., data splits, hyperparameters, how they were chosen, type of optimizer) necessary to understand the results?
    \item[] Answer: \answerYes{}
    \item[] Justification: Yes. Section~4 gives the core protocol, while the appendix and supplementary include per-dataset preprocessing, landmark budgets, frozen result snapshots, and the relevant configuration/runtime scripts.

\item {\bf Experiment statistical significance}
    \item[] Question: Does the paper report error bars suitably and correctly defined or other appropriate information about the statistical significance of the experiments?
    \item[] Answer: \answerNo{}
    \item[] Justification: No. The quality tables remain fixed-seed point estimates, and although the main runtime snapshots are medians over repeated warm-cache executions, we do not report error bars or significance tests; this limitation is stated in Section~4, Section~5, and the appendix.

\item {\bf Experiments compute resources}
    \item[] Question: For each experiment, does the paper provide sufficient information on the computer resources (type of compute workers, memory, time of execution) needed to reproduce the experiments?
    \item[] Answer: \answerNo{}
    \item[] Justification: No. The appendix reports the workstation type, memory, and representative runtime tables, but the paper does not provide a full audit of total project compute across all exploratory runs.
    
\item {\bf Code of ethics}
    \item[] Question: Does the research conducted in the paper conform, in every respect, with the NeurIPS Code of Ethics \url{https://neurips.cc/public/EthicsGuidelines}?
    \item[] Answer: \answerYes{}
    \item[] Justification: Yes. The work uses public benchmark datasets, releases only anonymous supplementary materials, and we are not aware of any aspect of the study that conflicts with the NeurIPS Code of Ethics.

\item {\bf Broader impacts}
    \item[] Question: Does the paper discuss both potential positive societal impacts and negative societal impacts of the work performed?
    \item[] Answer: \answerYes{}
    \item[] Justification: Yes. Section~5 notes the positive accessibility benefit of lower-latency exploratory analysis on commodity hardware and the negative risk that embeddings can be over-interpreted in scientific or biomedical workflows.
    
\item {\bf Safeguards}
    \item[] Question: Does the paper describe safeguards that have been put in place for responsible release of data or models that have a high risk for misuse (e.g., pre-trained language models, image generators, or scraped datasets)?
    \item[] Answer: \answerNA{}
    \item[] Justification: Not applicable. The paper does not release high-risk generative models, scraped data, or other assets that would require special misuse safeguards.

\item {\bf Licenses for existing assets}
    \item[] Question: Are the creators or original owners of assets (e.g., code, data, models), used in the paper, properly credited and are the license and terms of use explicitly mentioned and properly respected?
    \item[] Answer: \answerNo{}
    \item[] Justification: No. The paper cites the original datasets and baseline methods, but it does not enumerate the licenses or terms of use for every external asset in the manuscript or supplementary package.

\item {\bf New assets}
    \item[] Question: Are new assets introduced in the paper well documented and is the documentation provided alongside the assets?
    \item[] Answer: \answerYes{}
    \item[] Justification: Yes. The anonymous supplementary package documents the released FastUMAP code, configs, scripts, and lightweight result snapshots in \texttt{REPRODUCE\_ANON.md}.

\item {\bf Crowdsourcing and research with human subjects}
    \item[] Question: For crowdsourcing experiments and research with human subjects, does the paper include the full text of instructions given to participants and screenshots, if applicable, as well as details about compensation (if any)? 
    \item[] Answer: \answerNA{}
    \item[] Justification: Not applicable. The paper does not involve crowdsourcing or experiments with human subjects.

\item {\bf Institutional review board (IRB) approvals or equivalent for research with human subjects}
    \item[] Question: Does the paper describe potential risks incurred by study participants, whether such risks were disclosed to the subjects, and whether Institutional Review Board (IRB) approvals (or an equivalent approval/review based on the requirements of your country or institution) were obtained?
    \item[] Answer: \answerNA{}
    \item[] Justification: Not applicable. The paper does not involve human-subject research requiring IRB or equivalent review.

\item {\bf Declaration of LLM usage}
    \item[] Question: Does the paper describe the usage of LLMs if it is an important, original, or non-standard component of the core methods in this research? Note that if the LLM is used only for writing, editing, or formatting purposes and does \emph{not} impact the core methodology, scientific rigor, or originality of the research, declaration is not required.
    \item[] Answer: \answerNA{}
    \item[] Justification: Not applicable. LLMs are not an important, original, or non-standard component of the method or experiments.

\end{enumerate}